%% file: main_ecai.tex
\documentclass{article}

\usepackage{graphicx}
\usepackage{latexsym}

\usepackage{wrapfig}
\usepackage[utf8]{inputenc} 
\usepackage{hyperref}       
\usepackage{url}            
\usepackage{booktabs}       
\usepackage{amsfonts}
\usepackage{nicefrac}       
\usepackage{xcolor}         
\usepackage{amsthm,amsmath,stmaryrd}
\usepackage{graphicx}

\usepackage{pgfplots,wrapfig}
 
\pgfplotsset{compat = newest}

\newcommand{\N}{Noothigattu et al. \cite{N}~}

\usepackage{enumitem}

\begin{document}
\newtheorem{thm}{Theorem}[section]
\newtheorem{assumption}{Assumption}
\newtheorem{definition}{Definition}[section]
\newtheorem{corollary}[thm]{Corollary}
\newtheorem{lemma}{Lemma}
\newtheorem{thesis}{Thesis}
\newtheorem{observation}{Observation}
\newtheorem{proposition}[thm]{Proposition}

\newtheorem{problem}{Problem}
\newtheorem{conjecture}{Conjecture}
\newtheorem{axiom}{Axiom}
\newtheorem{remark}{Remark}

\newtheorem{example}{Example}[section]
\newtheorem{claim}{Claim}[section]


\title{No Agreement Without Loss:\\
Learning and Social Choice in Peer Review}

\author{Pablo Barcel\'o$^{123}$, Mauricio Duarte$^{4}$, Cristobal Rojas$^{13}$, Tomasz Steifer$^{125}$ }

\date{%
	$^1$Institute for Mathematical and Computational Engineering, PUC, Chile\\%
	$^2$Instituto Milenio Fundamentos de los Datos, Chile\\
	$^3$Centro Nacional de Inteligencia Artificial, Chile\\
	$^4$Universidad Andres Bello, Chile\\
	$^5$Institute of Fundamental Technological Research, Polish Academy of Sciences\\%
 }

\maketitle

\subsubsection*{Abstract}
In peer review systems, reviewers are often asked to evaluate various features of submissions, such as technical quality or novelty. A score is given to each of the predefined features and based on these the reviewer has to provide an overall quantitative recommendation. It may be assumed that each reviewer has her own mapping from the set of features to a recommendation, and that different reviewers have different mappings in mind. This introduces an element of arbitrariness known as commensuration bias. In this paper we discuss a framework, introduced by Noothigattu, Shah and Procaccia, and then applied by the organizers of the AAAI 2022 conference. Noothigattu, Shah and Procaccia proposed to aggregate reviewer's mapping by minimizing certain loss functions, and studied axiomatic properties of this approach, in the sense of social choice theory. We challenge several of the results and assumptions used in their work and report a number of negative results. On the one hand, we study a trade-off between some of the axioms proposed and the ability of the method to properly capture agreements of the majority of reviewers. On the other hand, we show that dropping a certain unrealistic assumption has dramatic effects, including causing the method to be discontinuous.

\input{text_ecai.tex}
\input{appendix_ecai.tex}

\bibliography{main}

\end{document}

%% file: text_ecai.tex
\section{Introduction}

\paragraph{\underline{Context}.}
Peer review comes with various issues and problems, which might be tackled from formal and computational perspectives, see, e.g., \cite{survey} for a survey. Previous research pointed out to a significant level of arbitrariness both in grant \cite{graves2011funding,erosheva2021zero} and conference peer review \cite{franccois2015arbitrariness}. One of the apparent influences on arbitrariness, called \text{commensuration bias} \cite{lee2015commensuration}, appears when reviewers are asked to provide overall scores on top of a set of incomparable scores given to specific criteria (such as novelty or methodological soundness). However, reviewers have different preferences when evaluating proposals or papers \cite{hug2022peers,shah2018design}. Some reviewers may prioritize novelty, while others may value technical quality more. As a result, the acceptance of a paper that reports well-substantiated but incremental results may depend on the group of reviewers to which the paper is assigned. 

Recently, \N have proposed 
a novel way of looking at this problem by combining approaches from machine learning and social choice theories.  
This framework is defined in terms of a set $\mathcal{R}$ of reviewers and  
a set $\mathcal{P}$ of papers. 
Each reviewer $i\in \mathcal{R}$ reviews each paper $a\in \mathcal{P}$ by providing a {\em score vector}  $\bar{x}_{ia}\in [0,10]^d$ and 
an {\em overall recommendation} $y_{ia}=h_i(\bar{x}_{ia})\in [0,10]$, for $h_i : [0,10]^d \to [0,10]$ a monotonic mapping used by $i$ to map score vectors to 
recommendations. 
One then aims at learning a single monotonic mapping $\hat h : [0,10]^d \to [0,10]$, representing the aggregation of the $h_i$'s, 
by minimizing an $L(p,q)$ loss function. 

Ideally, one would like to choose the parameters $p,q$ based on the data. However, for peer reviewing there is no underlying ground truth to adjust to, so other methods are needed. \N propose taking the perspective of social choice theory, and develop desirable properties, or {\em axioms}, that restrict the choice of parameters $p,q$ that can be used. These axioms are  
{\em consensus} (if all reviewers give the same recommendation to a paper, then the aggregate mapping should give the same recommendation to it), {\em efficiency} (if the set 
of recommendations given to a paper {\em dominates} the set of recommendations given to another paper, then the aggregate mapping should give the first paper a score at least as high as the second one), and {\em strategy-proofness} 
(reviewers should have no incentive to misreport their recommendations).   

The main result of \N is obtained under the following rather strong {\em Objectivity Assumption}:  for each paper $a$, all reviewers coincide in the score vector they assign to $a$. We call the setting obtained under this assumption the {\em hidden scores setting}. In this context, the only choice of parameters that satisfies all the axioms above is $p = q = 1$. 

The learned mapping can be used in practice to generate alternative overall recommendations for each review based on the criteria scores. These additional ratings can then be utilized by decision-makers, such as program committees for conferences or expert panels for funding agencies, to make more informed decisions. In fact, as mentioned in \cite{survey}, the 
organizers of the AAAI 2022 conference used this approach to "identify reviews with significant commensuration bias."


 \paragraph{\underline{Our results}.} 
 We present a body of negative results that show limitations of the discussed method and call into question whether it should be used in practice.

The results of \N suggest the $L(1,1)$ loss function as the best option for practical use in peer review aggregation, at least in the hidden score setting. We challenge this by arguing that $L(1,1)$ might be too crude to correctly differentiate between papers which should be ranked differently according to the expert opinions. We address this issue by proposing a new axiom, {\em consistency}. Can consistent solutions be obtained in the framework of \cite{N}? We show that this is possible for any $p,q\in(1,\infty)$ with $p = q$. We also show that consistency comes with a specific trade-off---a reasonable aggregation method cannot be consistent and strategy-proof at the same time. 


We then move to study the theoretical properties of the setting once the Objectivity Assumption is removed, that is, in the {\em non-hidden scores 
setting}. 
While the original setting does not require this assumption to hold in order to obtain aggregated solutions for papers, we show 
that many of its good axiomatic properties are lost if the Objectivity Assumption is dropped.  
First of all, the axioms as developed by \N no longer hold (even in the case $p = q = 1$).
The reason underlying this fact is that such axioms simply do not take 
score vectors into consideration, and hence they do not convey all the necessary information for expressing interesting properties in the non-hidden setting. 
We show that natural modified versions of consensus and efficiency are satisfied for certain loss functions, but the same holds neither 
for consistency nor for strategy-proofness. In fact, we show that in this general setting the $L(p,q)$ method fails to be strategy-proof in a very dramatic way---a malicious reviewer can obtain a large gain by an arbitrarily small misreporting of his beliefs. In particular, the method loses its stability in the sense that it is not even continuous with respect to the data.  

\paragraph{\underline{Related work}.} 

The idea of using loss functions to aggregate values 
was previously studied in the aggregation function theory 
 under the name of penalty-based functions; see, e.g., \cite{grabisch_marichal_mesiar_pap_2009,calvo2010aggregation,bustince2017definition}. We note that the framework of this theory is different from the one presented in this paper. Moreover, to the best of our knowledge, the $L(p,q)$ loss functions were not studied in this context. The idea of combining learning and aggregation in peer review was also pursued recently by other authors (see \cite{melnikov2016learning})---although with a goal of learning an aggregation function whose recommendations match that of the empirical data, instead of focusing on axiom satisfaction. 

Impossibility results on strategy-proofness in the area of 
social choice date back to the classic Gibbard-Satterthwaite Theorem 
\cite{gibbard1973manipulation, satterthwaite1975strategy}
and many works that came later, including Moulin's characterization of strategy-proofness for single-peaked utilities.
Relaxed versions of strategy-proofness were proposed as means to avoid these impossibility results and even quantify the degree to which strategy-proofness is violated \cite{carroll2011quantitative,birrell2011approximately, azevedo2019strategy, mennle2021partial,li2022differentially}. For instance, such modification may require that the voter has an incentive to report opinions close to his own (as opposed to exactly true opinions in standard strategy-proofness). 

\section{Framework} 
\label{sec:framework} 
We start by presenting more formally the framework of \N 
Consider a set $\mathcal{R}$ of $n$ reviewers and a set $\mathcal{P}$ of $m$ papers\footnote{The word 'paper' can be replaced by other, such as 'proposal', depending on the application.} . Each reviewer $i$ has to evaluate every paper $a$ \footnote{This assumption is somehow unrealistic but this is not an issue here, since our main results are negative.}  by providing:

\begin{enumerate}[wide, labelwidth=!, labelindent=0pt]
    \item a \textit{score} vector $\bar{x}_{ia}\in [0,10]^d\subset \mathbb{R}^d$ associated to each of $d$ different evaluation criteria, and 
    \item a recommendation $y_{ia}\in [0,10] \subset \mathbb{R}$, reflecting the summary evaluation of the paper.
\end{enumerate}

The problem consists of \textit{aggregating} reviewers' opinions into an overall recommendation $s_a \in [0,10]$ that reflects their opinions in an effective and fair manner. 

Recall from the introduction that the hypothesis class $\mathcal{H}$ corresponds to the family of {\em monotonic functions}, i.e., functions $h : [0,10]^d \to [0,10]$ satisfying: 
$$\bar x \leq \bar y \ \ \Longrightarrow \ \ h(\bar x) \leq h(\bar y),$$
where $\bar x \leq \bar y$ is a shortening for stating that the $i$-th component $x[i]$ of $\bar x$ is smaller or equal than the $i$-th component $y[i]$ of $\bar y$, for each $i\in [1,d]$.  It is also assumed that each reviewer $i \in \mathcal R$ has her own monotonic function $h_i : [0,10]^d \to [0,10]$ in mind, which she uses to rate papers, i.e., 
$y_{ia} = h_i(\bar{x}_{ia})$. This means, in particular, that if $\bar{x}_{ia} = \bar{x}_{ib}$, for papers $a,b \in \mathcal P$, then $y_{ia} = y_{ib}$.   

\paragraph{\underline{The $L(p,q)$-aggregation method}.}
The way in which the aggregated score $s_a$ for paper $a \in \mathcal{P}$ is obtained consists of the following two steps. 

\begin{enumerate}[wide, labelwidth=!, labelindent=0pt]
\item {\em Empirical Risk Minimization (ERM).} 
In this first step one finds a {\em monotonic} function that minimizes the {\em $L(p,q)$ loss} on the data, 
for $p,q \in [1,\infty)$.
In particular, one aims at computing the values $\hat{h}(\bar{x}_{ia})$, for $i \in \mathcal{R}$, $a \in \mathcal{P}$, and
$\hat h \in {\rm argmin}_{h \in \mathcal{H}} \, L(p,q)$, where:
\[ \label{eq:lpq}
L(p,q) \ := \ \bigg[\,\sum_{i \in \mathcal{R}} \, \bigg( 
\, \sum_{a \in \mathcal{P}} \,  |y_{ia} - h(\bar{x}_{ia})|^{{p}} \bigg)^{\nicefrac{q}{p}} 
\, \bigg]^{\nicefrac{1}{q}}\, .
\]
That is, the $L(p,q)$ loss is obtained by first taking for each reviewer the $L_p$ norm 
on the error of $h$ with respect to the reviewer’s recommendations, followed by computing 
the $L_q$ norm over these errors.  We emphasize that when in this setting one speaks about finding a ``function'' $\hat{h}$, what is meant is simply finding the finitely many values that $\hat{h}$ takes on the available score vectors. 
 
 Since the solution to this problem is not unique for certain values of $p$ and $q$, \N propose a way to break ties by choosing the $\hat h$ for which the solution vector $\hat{h}(\bar{x}_{ia})_{ia}$ has the smallest $L_2$ norm, and show that this method singles out a unique solution. We denote by $\hat{y}_{ia}=\hat{h}(x_{ia})$ the values resulting from this step. 

Note that when the available data $(\bar{x}_{ia},y_{ia})$, for $i,j \in \mathcal{R}$ and $a,b \in \mathcal{P}$, 
is monotonic {\em across} reviewers---i.e., $y_{ia}\geq y_{jb}$ whenever $\bar{x}_{ia}\geq \bar{x}_{jb}$---then (assuming all score vectors are different) the solution $\hat{h}(\bar{x}_{ia})=y_{ia}$ achieves zero loss. However, monotonicity across reviewers does not follow from the hypothesis of individual reviewer's monotonicity. Thus the main effect of this step is to ``modify'' the recommendations $y_{ia}$ into $\hat{h}(\bar{x}_{ia})$ in order to ensure monotonicity across reviewers.


\medskip 

\item 
{\em  Aggregation.} Second, based on the modified recommendations $\hat{y}_{ia}$ only, for $i \in \mathcal{R}$, 
one computes an aggregated solution $s_a \in [0,10]$ for each $a \in \mathcal{P}$. 
One can do this by following the same approach as in the ERM step, namely, finding the vector of solutions $(s_a)_{a \in \mathcal{P}}$  
that minimizes the $L(p,q)$ loss given by:
\[ \label{eq:lpqs}
L(p,q) \ := \ \bigg[\,\sum_{i \in \mathcal{R}} \, \bigg( 
\, \sum_{a \in \mathcal{P}} \,  |\hat{y}_{ia} - s_a|^{{p}} \bigg)^{\nicefrac{q}{p}} 
\, \bigg]^{\nicefrac{1}{q}}\, .
\]
Note that score vectors $\bar{x}_{ia}$ play no role in this step. 
\end{enumerate} 

\paragraph{\underline{Axioms}.} Unlike many other problems in machine learning,  
there is no directly available ground truth for the peer review process. 
The way in which this problem is approached by \N is by specifying certain socially desirable properties, in the form of axioms, and then studying 
for which combinations of $p$ and $q$ the solution provided by the empirical risk minimization of 
the $L(p,q)$ loss satisfies the axioms. We recall these axioms below. 


The first axiom states that if all reviewers give the same overall 
recommendation to a paper, then 
the aggregated solution should coincide with that value.

\begin{axiom}[{\sc Consensus}]
\label{ax.consensus}
If for a paper $a \in \mathcal P$ we have $y_{ia} = y$, for each $i \in \mathcal R$, then $s_a = y$. 
\end{axiom} 

The second principle requires defining the concept of {\em domination} among papers. We say that paper $a$ {\em dominates} paper $b$, if there is a permutation $\pi : {\mathcal R} \to {\mathcal R}$ of the set of reviewers such that $y_{ia} \geq y_{\pi(i)b}$ for every reviewer $i \in \mathcal R$. 
%
The second axiom states that if $a$ dominates $b$, then the aggregated solution $s_a$ should be at least as high as the aggregated solution $s_b$. 


\begin{axiom}[{\sc Efficiency}]
\label{ax.efficiency}
Whenever paper $a$ dominates paper $b$, the solution satisfies $s_a\geq s_b$. 
\end{axiom}

The third principle introduced by \N is based on the notion of {\em strategy-proofness}, a game-theoretic concept which plays a fundamental role in social choice theory \cite{CSC-book,DBLP:conf/ijcai/Xu0SS19}. 
The application of strategy-proofness in the current setting formalizes the idea that we do not want individual reviewers to be able to manipulate the solution (i.e., to bring the solution closer to their own opinion) by misreporting their true recommendations. 
It is motivated by the discovery of strategic behavior in peer reviewing \cite{fraud,DBLP:conf/aaai/StelmakhSS21,DBLP:journals/corr/abs-2201-10631}. 

\begin{axiom}[{\sc Strategy-proofness}]
\label{ax.strategy-proof}
For each reviewer $i$ and each misreported recommendations vector $(y'_{ia})_{a \in \mathcal{P}}$,
the manipulated solution $(s'_a)_{a \in \mathcal{P}}$ satisfies $$\|(s'_a)_{a \in \mathcal{P}} - (y_{ia})_{a \in \mathcal{P}}\|_2 \ \geq \  \|(s_a)_{a \in \mathcal{P}}-(y_{ia})_{a \in \mathcal{P}}\|_2.$$
\end{axiom}


\subsection{The Objectivity Assumption} 
Theoretical results of \N are obtained under the following strong {\em Objectivity 
Assumption}: 
For each paper $a \in \mathcal P$, all score vectors given by reviewers coincide, i.e., 
$$\bar{x}_{ia} = \bar{x}_{ja}=\bar{x}_a, \ \ \text{for each $i,j \in \mathcal R$.}$$ 
This has interesting implications: since now we have only one score vector per paper, the ERM 
step becomes finding a single value $\hat{y}_a=\hat{h}(\bar{x}_a)$, for each $a \in \mathcal{P}$. But this now makes the Aggregation step of the method trivial, since we can choose $s_a=\hat{y}_a$ for every $a\in \mathcal{P}$ and achieve zero loss in this step. 
As shown by \N, under this strong assumption it is possible to obtain a neat characterization of which $L(p,q)$ solutions satisfy the axioms.  

\begin{thm}[\N] \label{theo:previous} 
Under the Objectivity Assumption, the $L(p,q)$-aggregation method satisfies 
{\sc Consensus}, {\sc Efficiency}, and 
{\sc Strategy-proofness} iff $p = q = 1$. 
In particular: 
\begin{itemize}
\item When $p \in (1,\infty)$ and $q = 1$, {\sc Efficiency} is violated.
\item When $q \in (1,\infty)$, {\sc Strategy-proofness} is violated.  
\end{itemize} 
\end{thm} 


\vspace{-0.3cm}

\section{The Hidden Scores setting}
\label{sec:main-results} 

As mentioned in the introduction, one of our motivations is to understand how these axioms can be calibrated in order to produce 
more robust and flexible results in the learning process. In this section 
we will be analyzing the results of \N, and thus for the time being we work under the Objectivity Assumption. 

We start with a simple yet important observation regarding this setting, which will guide all our subsequent discussions. The only role played by score vectors in the aggregation is for defining monotonicity conditions in the ERM step of the aggregation. This means that the minimizer of the $L(p,q)$ norm taken among all functions $h$ may be different from the minimizer taken from the class of monotonic functions. However, this does not happen under the Objectivity Assumption as we assume that reviewers use monotonic mappings. This motivates the following observation.


\begin{observation}\label{hidden}
Under the Objectivity Assumption, the output of $L(p, q)$ aggregation method is independent of score vectors. 
\end{observation}

This means that, in the Objective Setting the score vectors can be ignored. We therefore call this the {\em hidden score setting}.

\subsection{Consistency}

Recall from Theorem \ref{theo:previous} that, under the Objectivity Assumption, the only solution method that satisfies all three axioms is $L(1,1)$. As shown by \N, in such case we have that for each paper $a$ the solution $s_a$ is the left-median of the values $y_{ia}$, for $i \in \mathcal{R}$. The problem with this solution is, of course, that it might be insensitive to situations in which 
most reviewers agree that one of the papers is largely better than some other paper. 
To illustrate this, consider the following recommendation vectors:

\begin{itemize}
\item two papers with four reviewers each and  recommendations given by $[8,\, 8,\, 3,\, 8]$ and $[8,\, 8,\, 9, \,8]$;
\item two papers with three reviewers each and recommendations given by $[3,\, 1,\, 4]$ and $[2, 3, 9]$.
\end{itemize}

In both examples, despite a clear preference for one of the papers, the left-median solution would assign them the same score: $s=8$ in the first example and $s=3$ in the second. We believe that a strong case can be made that, in these examples, it would not seem very consistent to be unable to discriminate between papers based on their aggregated score. In order to address this issue we propose the following axiom, which strengthens {\sc Efficiency}. Recall that paper $a$ {\em dominates} paper $b$, if there is a permutation $\pi : {\mathcal R} \to {\mathcal R}$ of the set of reviewers such that $y_{ia} \geq y_{\pi(i)b}$ for every reviewer $i \in \mathcal R$. Similarly, we say that paper $a$ {\em strictly dominates} paper $b$, if at least one of these inequalities is strict.


\begin{axiom}[{\sc Consistency}] Axiom {\sc Efficiency} holds 
and whenever paper $a$ strictly dominates paper $b$, then $s_a>s_b$.
\end{axiom}
\smallskip 

Note that $L(1,1)$ solutions are efficient but not consistent. 
We observe, in contrast, that $L(p,q)$ solutions always satisfy {\sc Consensus} and {\sc Consistency} when $p=q$, for any $p > 1$. 


\begin{thm}\label{the-consistency}
For all $p,q \in (1,\infty)$, the solutions based on the $L(p,q)$-aggregation method satisfy {\sc Consensus}, and for $p=q > 1$, they satisfy {\sc Consistency}.\end{thm} 

\begin{proof}

We first show Consensus. Assume then that $y_{ib} = y$ for some paper $b\in\mathcal{P}$ and all reviewers $i\in\mathcal{R}$. By looking at the formula for the loss function (note that $p,q\neq \infty$): 
\[
L(p,q) \ := \ \bigg[\,\sum_{i \in \mathcal{R}} \, \bigg( 
\, \sum_{a \in \mathcal{P}} \,  |y_{ia} - h(\bar{x}_{ia})|^{{p}} \bigg)^{\nicefrac{q}{p}} 
\, \bigg]^{\nicefrac{1}{q}}\,,
\]
one observes that, by the objectivity assumption, it holds that $|y_{ib}-h(x_{ib})|=|y-h(x_{1b})|$ for all $i$. In order to minimize the loss function, one must choose $\hat y_{ib}=y$ and $s_b=y$, which makes all such terms equal to zero, while respecting monotonicity. 

To show consistency, first note that minimizing $L(p,q)$ is equivalent to minimizing $L(p,q)^q$ for any $q\in [1,\infty)$, so we will focus on the latter. For the particular case when $p=q$ we have:
\[
L(p,p)^p = \sum_{i\in\mathcal{R}}\sum_{a\in\mathcal{P}}  |y_{ia}-s_a|^p= \sum_{a\in\mathcal{P}} \sum_{i\in\mathcal{R}} |y_{ia}-s_a|^p.
\]

Thus, when $p=q$,  we can solve the problem separately for each paper, and the solution for a given paper $a$ is given by
$$
s_a=\text{arg}\min_s \sum_{i\in \mathcal{R}}|y_{ia}-s|^p.
$$

The proof of consistency will rely on the following simple lemma. Given a vector $y\in\mathbb{R}^n$, let 
\[E_y(s)=\sum^{n}_{i=1}|y_i-s|^p \quad \text{ and } \quad  s^*(y)=\text{arg}\min_s E_y(s).\]

 \begin{lemma}\label{lemma-equality}
 For $y\in \mathbb{R}^n$, and $p>1$, the function $E_y$ is differentiable, and $E'_y$ is strictly increasing. In particular, 
 \begin{equation}\label{lemma}
 \sum_{y_i\leq s^*(y)}(s^*(y)-y_i)^{p-1}=\sum_{s^*(y)\leq y_j} (y_j-s^*(y))^{p-1}.
 \end{equation} 
 \end{lemma}
 
\begin{proof}
Since $p>1$, $E_y$ is differentiable and at its minimum we have $E'_y(s^*)=0$. Differentiating $E_y(s)$ with respect to $s$ we get
 \begin{eqnarray*} E'_y(s)  = &\displaystyle{\sum_{i=1}^n-p|y_i-s|^{p-1}{\rm sign}(y_i-s)} \\  = &\displaystyle{ p\left(\sum_{y_i \leq s}(s-y_i)^{p-1} - \sum_{s\leq y_j}(y_j-s)^{p-1}\right)}.
 \end{eqnarray*} 
 $E'_y(s)$ is therefore a strictly increasing function of $s$ passing through 0 at $s=s^*(y)$. The equality in the Lemma immediately follows. 
 \end{proof}

We can now show consistency (which includes efficiency). Let $y',y\in\mathbb{R}^n$ be the recommendation vectors for two different papers such that $y'$ dominates $y$. If the domination is not strict, then the recommendations are actually the same (after permutation), and in this case their solutions will also be the same (since both minimize the same expression), so the condition for efficiency is satisfied. We now show that whenever $y'$ strictly dominates $y$, we have that $s^*(y')> s^*(y)$. After permutation, we can assume that $y'$ equals $y$ with some coordinates strictly increased. It follows that either
\[
\sum_{y'_i \leq s^*(y)}(s^*(y)-y'_i)^{p-1}<\sum_{y_i\leq s^*(y)}(s^*(y)-y_i)^{p-1}
\]
or
\[
\sum_{s^*(y)\leq y_j'}(y'_j-s^*(y))^{p-1}>\sum_{s^*(y)\leq y_j}(y_j-s^*(y))^{p-1}
\]
or both. Subtracting these inequalities and using the computation in Lemma \ref{lemma}, we obtain  that $E_{y}'(s^*(y))=0 > E_{y'}'(s^*(y))$, which in turn shows that $s^*(y')>s^*(y)$, since $E'_{y'}$ is strictly increasing. This finishes the proof.
\end{proof}




%




We now turn to the problem of strategy-proofness. Let us first recall that, by the results of \N,  $L(p,q)$-aggregation is not strategy-proof when $q\in(1,\infty)$. As shown in the next result, this is a consequence of a more general conflict between consistency and strategy-proofness. The following result applies generally to aggregation methods of the form $s: (y_{ia})_{ia} \mapsto (s_a)_a$ (of which $L(p,q)$ aggregation is a special case). 


\begin{proposition}\label{prop.consistency-notsproof}
If an aggregation method is continuous and satisfies {\sc\ Consensus} and {\sc\  Consistency}, then it does not satisfy 
{\sc Strategy-proofness}.
\end{proposition}
\begin{proof}

\begin{figure}
  \centering
\begin{tikzpicture}  
        \begin{axis}[
        xtick={0.5,1.5,3.5},
    xticklabels={$y'_{1a}$,$s'_a$,$y'_{2a}$},
        xmin=0, xmax=4,      
        axis x line=bottom,
    hide y axis,    
    ymin=0,ymax=5,
        xlabel={recommendations},
    width = 0.35\textwidth,
    height = 0.3\textwidth, 
    legend pos=north west,]]
        
\addplot[
    only marks,
    color=blue,
    mark=*,
    mark size=1.4pt,
    ] coordinates {(1.5,2)(2.2,2)(3.5,2)(0.5,0.7)(1.5,0.7)(3.5,0.7)};

\node[above,black] at (1.5,3) {$y_{1b}=y_{1a}=y_{2b}$};
\node[above,black] at (1.5,2.15) {$\downarrow$};
\node[above,black] at (2.2,2) {$s_{a}$};

\node[above,black] at (3.5,2) {$y_{2a}$};
    \end{axis}
\end{tikzpicture}
\caption{The example in the proof of Proposition 3.2.}
\label{fig:nonsp}
\end{figure}
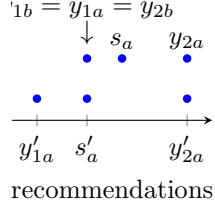

We consider two papers $p = a,b$ and two reviewers $i = 1,2$. We will construct a data set, which allows for manipulation, contradicting {\sc strategy-proofness}. For $t
\in [0,1]$, consider data sets $$D_t \ = \{y_{1a}=4,y_{1b}=4+2t,y_{2b}=4+2t,\ y_{2a}=6\}.$$ 
The solution $s^{t}$ for data set $D_t$ is continuous as a function of $t$ by assumption. For $t=0$, by consensus and consistency we have $s^t_b = 4< s^t_a$. Similarly, for $t=1$ we have $s^t_b =6 > s^t_a$. By continuity, there is $u\in (0,1)$ such that $s^{u}_a = s^{u}_b = 4+2u$. Next, consider the following data set: $y_{1a}=y_{1b}=y_{2b}=s^{u}_a,y_{2a}=6$, and a maliciously modified recommendation $y'_{1a}=4$. Note that the malicious data set corresponds to $D_u$. It follows by consensus and consistency that the respective solutions $s$ and $s'$ to these data sets satisfy:
$$
y_{2a} \ge s_a >  s'_a = s^u_a = y_{1a} > y'_{1a}.
$$
Hence $|y_{1a}-s_a| > |y_{1a}-s'_a|=0$, which contradicts {\sc Strategy-proofness} for reviewer 1, which completes the proof.
\end{proof}

\section{The Non-Hidden Scores setting}
\label{sec:understanding}

We now discuss the implications of dropping the Objectivity Assumption. We call this the {\em non-hidden scores setting}. As suggested by \N, their framework does not require this assumption to hold, only their axiomatic characterization does (see Theorem \ref{theo:previous}). In this section we study the status of these axioms in the non-hidden scores setting. 

\subsection{Failure of the axioms}
 We start with a simple yet important observation: without any modification, none of the axioms is satisfied, regardless of the values for $p$ and $q$. For {\sc Consensus} and {\sc Efficiency} (and hence {\sc Consistency}), the main reason for this failure is that, as stated, the axioms are defined in terms of recommendations alone---score vectors are not taken into account. 

\begin{thm}
\label{prop.consistency0}
For all $p,q \in [1,\infty)$, the $L(p,q)$-aggregation method in the non-hidden scores setting satisfies neither {\sc Consensus} nor {\sc Efficiency} nor {\sc Strategy-Proofness}.  
\end{thm}
\begin{figure}
\centering
\begin{tikzpicture}  
        \begin{axis}[
        xmin=0, xmax=5,      
        ymin=0,ymax=3.2, 
        ylabel={recommendations},
    xlabel={scores},
    width = 0.35\textwidth,
    height = 0.2\textwidth, 
    legend pos=north west,]]
        
\addplot[
    only marks,
    color=blue,
    mark=*,
    mark size=1.4pt,
    ] coordinates {(1,1)(2,2)};
\addplot[
    only marks,
    color=black,
    mark=*,
    mark size=1.4pt,
    ] coordinates {(3,1)(4,2)};
\node[above,lightgray] at (1,1) {$y_{1a}$};
\node[above,gray] at (2,2) {$y_{1b}$};
\node[above,lightgray] at (3,1) {$y_{2a}$};
\node[above,gray] at (4,2) {$y_{2b}$};
    \end{axis}
\end{tikzpicture}
\caption{Counter-example for the proof of Theorem \ref{prop.consistency0}.}
\label{fig_counter_41}
\end{figure}
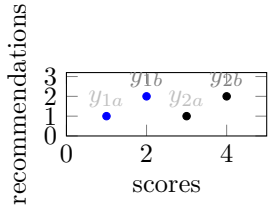
\begin{proof}
We prove that {\sc Consensus} fails by exhibiting a counterexample in which score vectors consist of a single criterion. The reviews satisfy $x_{1a}<x_{1b}<x_{2a}<x_{2b}$ and $y_{1a}=y_{2a}<y_{1b}=y_{2b}$, as in Figure \ref{fig_counter_41}. Clearly, the ERM step of the method satisfies 
$\hat{y}_{1b} =  \hat{y}_{2a} <  \hat{y}_{2b} = y_{2b}$,
even if $p=1$, since ties are broken by minimizing the $L_2$ norm. Hence $s_b<y_{2b}$ and {\sc Consensus} is not satisfied. 

In this same situation  {\sc Strategy-proofness} is violated. Indeed, for $p=1$ (when solution is left-median), by misreporting $y'_{2a}$ slightly larger than $y_{2a}$, she gets $s'_a=s_a$ and $s'_b > s_b$, so that the solution to misreported data is closer to the actual recommendations by reviewer 2. For $p>1$, by misreporting $y'_{2b}>y_{2b}$ the solution for $b$ gets moved up,  $s'_b > s_b$, making it closer to $y_{2b}$, while keeping $s_a'=s_a$.

Finally, we sketch the proof that {\sc Efficiency} fails. The following argument is carried out 
for $L(2,2)$-aggregation but it can easily be generalized to other $p,q$. We consider two papers $a,b$ reviewed by three reviewers. Suppose that $x_{1a}=x_{2a}=x_{1b}$ and assume that all the other score vectors are incomparable. Now, as a concrete example consider the following values $y_{1a}=y_{1b}=1$ and $y_{2a}=y_{2b}=4$ and $y_{3a}=y_{3b}=6$. Both papers dominate each other. The solutions obtained via $L(2,2)$-aggregation are as follows: 
\[s_a=((1+1+4)/3 + 4 + 6)/3 \ > \ s_b=(2+2+6)/3.\]
This happens because the aggregation method forces us to first aggregate $y_{1a}$, $y_{2a}$, and $y_{1b}$ into one value $\hat{y}_{1a}=\hat{y}_{1b}=\hat{y}_{2a}$, before carrying out the second step and aggregating modified recommendations into solutions. 
\end{proof}

\subsection{Axioms considering score vectors}

The previous result shows that the axiomatic behavior of the non-hidden scores setting can only be meaningfully studied in 
a scenario in which axioms also consider score vectors. 
This is what we do in this section; in fact, by modifying the axioms in a natural way 
we will see that {\sc Consensus} and {\sc Efficiency} can be recovered. Interestingly, this is not the case for {\sc Consistency}. More striking is the situation for {\sc Strategy-Proofness}, since its failure comes from the interaction between recommendations for \emph{different papers} enforced by the monotonicity restriction of the ERM step of the aggregation method. Surprisingly, as we will see below, when score vectors are taken into account, this failure becomes even more dramatic---a reviewer can now change the solution with an \emph{arbitrarily small} modification to her reported score vectors. 

\paragraph{\underline{Consensus and Efficiency with score vectors}.}
For {\sc Consensus}, it is natural to assume that the reviewers agree on both recommendations and score vectors.

\begin{axiom}[{\sc Consensus with score vectors}] 
If for a paper $a\in \mathcal{P}$ all the reviewers agree in their recommendations  $y_{ia}=y$ and in 
their score vectors $\bar{x}_{ia}=\bar{x}_a$, then $s_a=y$. 
\end{axiom}
We now easily get:
\begin{proposition}\label{prop.consensus1}
The $L(p,q)$-aggregation method, with $p,q\in[1,\infty)$, satisfies {\sc Consensus with score vectors}.
\end{proposition}
\begin{proof}
Suppose that for some paper $a$, some $x_a,\ y_a$ and every reviewer $i$, $x_{ia}=x_{a}$ and $y_{ia}=y_a$. Typically, minimizing the 
loss function will give us an aggregated value $\hat{y}_{ia}=y_a$ unless it is moved to allow for satisfying the monotonicity constraint. This happens only when there exist $j$ and $b$ such that either  $x_{jb} < x_a$ with $\hat{h}(x_{jb})> y_a$, or $x_{jb} > x_a$ with $\hat{h}(x_{jb})< y_a$. By monotonicity of reviewer $j$, in the first case we have $y_{jb}\leq y_{ja}=y_a$, and in the second $y_{jb}\geq y_{ja}=y_a$, so none of the mentioned conflicts of monotonicity can occur, and thus $\hat{y}_{ia}=y_a$.
Finally, we have to aggregate the values $\hat{y}_{1a},\ldots,\hat{y}_{na}$ to obtain a solution. This corresponds to aggregation under the Objectivity Assumption, and since the solution satisfies consensus in that case we have $s_a=y_a$.
\end{proof} 



To adapt efficiency to the non-hidden setting, we first need to extend the domination relation between papers to consider score vectors as well. We say that {\em paper $a$ dominates paper $b$ under score vectors}, if there is a permutation $\pi : \mathcal{R} \to \mathcal{R}$ on the set of reviewers such that $\bar{x}_{ia} \geq \bar x_{\pi(i)b}$ and $y_{ia} \geq y_{\pi(i)b}$, for each reviewer $i \in \mathcal{R}$. Thus, in the absence of score vectors, this notion coincides with the notion of dominance defined in Section \ref{sec:framework}. Then {\em $a$ strictly dominates $b$ under score vectors} 
if at least one of the dominations in the definition is strict.

\begin{axiom}[{\sc Efficiency with score vectors}] If paper $a$ dominates paper $b$ under score vectors, then $s_a \geq s_b$.  
\end{axiom}

The next statement follows easily from the monotonicity of solutions. It shows that efficiency with score vectors holds for  the same values of $p$ and $q$ as in the hidden scores setting.

\begin{proposition}\label{prop.efficiency1}
The solution based on the $L(p,q)$-aggregation method, for $p = q = 1$ and $p,q\in (1,\infty)$, satisfies {\sc Efficiency with score vectors}.
\end{proposition}
\begin{proof}
Without loss of generality we assume that the solution satisfies efficiency under the Objectivity Assumption. Fix two papers $a,b$ and assume for each reviewer $i$ we have $x_{ia}\geq x_{ib}$. By monotonicity, we have for each reviewer $\hat{y}_{ia}\geq\hat{y}_{ib}$. It follows that $s_a\geq s_b$.
\end{proof}

\paragraph{\underline{Consistency with score vectors}.}
We now turn to study consistency in the context of the non-hidden scores setting. 

\begin{axiom}[{\sc Consistency with score vectors}] 
If paper $a$ strictly dominates paper $b$ under score vectors, then $s_a > s_b$.  
\end{axiom}

Quite surprisingly, the following shows that {\sc Consistency with score vectors} is never satisfied by the $L(p,q)$ solution, regardless of $p$ and $q$. The result holds even if we require the reviewers to be strictly monotonic or the domination between papers to be strict for every reviewer, or both. 

\begin{thm}\label{thm.consistency1}
For any $p\in(1,\infty)$ and $q\in(1,\infty)$, there are papers $a,b$ such that $a$ strictly dominates  $b$ under score vectors and yet the $L(p,q)$ aggregated solution satisfies $s_a=s_b$.
\end{thm}
\begin{proof}
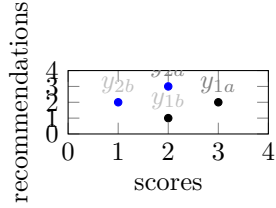
\begin{figure}
\centering
\begin{tikzpicture}  
        \begin{axis}[
        xmin=0, xmax=4,      
        ymin=0,ymax=4, 
        ylabel={recommendations},
    xlabel={scores},
    width = 0.35\textwidth,
    height = 0.2\textwidth, 
    legend pos=north west,]]
        
\addplot[
    only marks,
    color=blue,
    mark=*,
    mark size=1.4pt,
    ] coordinates {(1,2)(2,3)};
\addplot[
    only marks,
    color=black,
    mark=*,
    mark size=1.4pt,
    ] coordinates {(2,1)(3,2)};
\node[above,lightgray] at (1,2) {$y_{2b}$};
\node[above,gray] at (2,3) {$y_{2a}$};
\node[above,lightgray] at (2,1) {$y_{1b}$};
\node[above,gray] at (3,2) {$y_{1a}$};
    \end{axis}
\end{tikzpicture}
\caption{Counter-example from the proof of Theorem \ref{thm.consistency1}.}
\label{fig_proof_44}
\end{figure}
The counterexample consists of two papers $a,b$ and two reviewers. The number of criteria is $1$. Set the following scores: $x_{2b}<x_{2a}=x_{1b}<x_{1a}$ with recommendations being $y_{2b}=2=x_{1a}$, $y_{1b}=1$ and $y_{2a}=3$ (see Figure \ref{fig_proof_44}). After the first step of aggregation we get $\hat{y}_{1b}=\hat{y}_{2a}$ (because $x_{1b}=x_{2a}$). In fact, $\hat{y}_{1b}=2$. Also, since there are no additional constraints, we have $\hat{y}_{2b}=y_{2b}=y_{2a}=\hat{y}_{2a}$. Hence, the modified recommendations for both papers weakly dominate each other (being the same). We conclude that $s_a=s_b$, since {\sc Efficiency} holds under the Objectivity Assumption.
\end{proof}

\paragraph{\underline{Continuity}.} 
Finally, we show that strategy-proofness fails dramatically as a malicious reviewer can obtain a huge gain by an arbitrarily small misreporting of his beliefs. 

\begin{thm}\label{thm-strat-with-scores}
The $L(p,q)$-aggregation method is not continuous as a function of the score vectors. 
\end{thm}
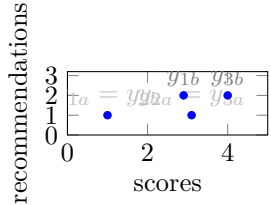
\begin{figure}
  \centering
\begin{tikzpicture}
        \begin{axis}[
        xmin=0, xmax=5,      
        ymin=0,ymax=3.2, 
        ylabel={recommendations},
    xlabel={scores},
    width = 0.35\textwidth,
    height = 0.2\textwidth, 
    legend pos=north west,]]
        
\addplot[
    only marks,
    color=blue,
    mark=*,
    mark size=1.4pt,
    ] coordinates {(1,1)(2.9,2)(3.1,1)(4,2)};

\node[above,lightgray] at (1,1) {$y_{1a}=y_{2b}$};
\node[above,lightgray] at (3.1,1) {$y_{2a}=y_{3a}$};
\node[above,gray] at (2.9,2) {$y_{1b}$};
\node[above,gray] at (4,2) {$y_{3b}$};
    \end{axis}
\end{tikzpicture}
\caption{An example showing influence of a malicious reviewer under $L(1,1)$-aggregation. 
}
\label{fig-strategyproof-with-vectors}
\end{figure}
\begin{proof}
Consider three reviewers and two papers with $x_{1a}=x_{2b}<x_{1b}<x_{2a}=x_{3a}<x_{3b}$ and $y_{1a}=y_{2b}=y_{2a}=y_{3a}<y_{1b}=y_{3b}$ (see Figure \ref{fig-strategyproof-with-vectors}). By monotonicity requirement, $\hat y_{1b}$ is forced to be strictly smaller than $y_{1b}$. In particular, $s_b$ will be strictly smaller than $y_{1b}$. However, the reviewer $1$ can maliciously misreport $x_{1b}$ so that $x_{1b}>x_{2a}=x_{3a}$. This would cause the monotonicity requirement to be trivially satisfied, allowing the modified solution $s'_b$ to become closer to $y_{1b}$, while having $s'_a=y_{1a}$, which shows that \ {\sc Strategy-proofness} is violated. Now, we may assume 
$x_{1b}=x_{2a}-\epsilon$ for an arbitrarily small $\epsilon$. At the same time, the gain obtained by the malicious reviewer, i.e, the difference $|s_b-s'_b|$, does not depend on $\epsilon$. Hence the solution is not continuous. 
\end{proof}

\section{Conclusions}

While the method proposed by \N does not require the Objectivity Assumption to hold, our results provide 
strong evidence that when this assumption is withdrawn the setting loses its good axiomatic behavior. These results are not only of theoretical interest as the method in question was reportedly used in practice by the organisers of a major AI conference \cite{survey}.

A particularly disturbing aspect of the setting without the Objectivity Assumption is its discontinuity, as stated in Theorem \ref{thm-strat-with-scores}. This means that a very small random perturbation in the criteria scores obtained by a paper $a$ can cause significant change to the ratings of an unrelated paper $b$. One of the main motivations for the discussed framework is to counteract the apparent arbitrariness in peer review. But discontinuity of the method introduces just a new type of arbitrariness to the decisions informed by its outputs.

Our results suggest a deep theoretical difficulty underlying criteria aggregation in peer review. However, they are not definite evidence for the impossibility of a satisfying solution. Indeed, one can note that failure of the axioms for the method of \N appears even in a toy scenario with just one criterion---which means they do not come specifically from dealing with mapping several criteria into one score, but rather are intrinsic to this particular method of aggregation.

This leads to the question of whether there is a meaningful modification of the discussed setting that remedies (at least partially) this bad behavior. A natural idea, already suggested by \N, is to limit the space of hypotheses by considering, e.g., only linear mappings. In such a case, the aggregation could involve the fitting of a linear mapping for each reviewer, followed by aggregation of these into one learned mapping. We briefly discuss such a framework in Appendix A and observe that it allows us to recover continuity and consistency. However, previous experiments (vide e.g. \cite{N}) suggest that reviewers' mappings are non-linear. Hence, such simplification might have undesirable consequences of its own and the linear framework should be considered simply as a proof-of-concept and motivation for further studies.

Finally, future research could involve rethinking of the problem to address the issue of miscalibration \cite{wang2018your}. Miscalibration bias comes from the apparent discrepancy between scales used by different reviewers when giving ratings. For example, Siegelmann \cite{siegelman1991assassins} reported that some reviewers of radiology papers were consistently giving lower-than-average scores to papers ('assassins'), while a different group of 'zealots' exhibited the opposite behavior. Again, this introduces arbitrariness of the outcomes of the peer review process. This issue will become especially important once we drop the assumption that all reviewers review all papers. A possible direction for future work would then include proposing a setting that addresses both the commensuration and miscalibration biases at the same time.

%% file: appendix_ecai.tex
\section*{Appendix A}
We discuss a modification of the setting here, based on the natural 
idea of using linear functions for explaining the behavior of each reviewer (as suggested in \cite{N}). 
Such linear functions are then aggregated into a single linear function that is used to compute the final score of each paper. This approach brings the following two immediate benefits over the non-hidden scores setting: (1) 
The obtained setting can easily be seen to be continuous as a function of the score vectors; and 
(2) in addition to providing us with an aggregated score vector for each paper, it also provides us with a function that explains the behavior of each reviewer, and that can potentially be applied to previously unseen data. 
    


\subsection{A linear functional setting}

We propose a simple aggregation scheme consisting of the following three steps. 

\begin{enumerate}[wide, labelwidth=!, labelindent=0pt]

 \item {\em  Linearization of the reviewers.} We fit a monotone linear function 
 $h_i:[0,10]^d\to\mathbb{R}$ separately for each reviewer $i\in\mathcal{R}$; e.g., by taking 
$h_i = \arg\min \sum_{a\in\mathcal{P}}|h(\bar{x}_{ia})-y_{ia}|^p$, 
where $h(\bar{x})$ ranges over all linear functions 
$c[{d+1}]+\sum_{j=1}^d c[j]\bar{x}_{ia}[j]$ with 
$(c[1],\dots,c[d+1]) \in (\mathbb{R}_{\geq 0})^{d+1}$, i.e., all coefficients in $\bar c$ are non-negative thus ensuring monotonicity of 
the obtained linear function. 
We denote by $\bar{c}_i$ the vector of coefficients that defines the linear function 
$h_i$. 
 
 \item {\em Aggregation of linear functions.} We now aggregate the $h_i$'s into a single linear function $\hat{h}$ defined by coefficients $\hat{c}[j]$ obtained by minimizing:
\[ \label{eq:lpq-linear}
L(p,q) \ := \ \bigg[\,\sum_{i \in \mathcal{R}} \, \bigg( 
\, \sum_{j \in [1,d+1]} \,  |c_i[j] - \hat{c}[j]|^{{p}} \bigg)^{\nicefrac{q}{p}} 
\, \bigg]^{\nicefrac{1}{q}}\, 
\]
Ties are broken by minimization of the $L(2)$ norm.
 
 \item {\em Aggregation of the score vectors.} Finally, we aggregate score vectors for each paper $a\in\mathcal{P}$ into an aggregated score vector $\bar{u}_a$. The solution $s_a$ for paper $a$ is $s_a=\hat{h}(\bar{u}_a)$. The aggregated score vectors $\bar{u}_a$ are obtained by minimization of global $L(p,q)$ loss functions applied to score vectors:
 \[
\bar{u}_a \ := \arg\min \ \bigg[\,\sum_{i \in \mathcal{R}} \, \bigg( 
\, \sum_{j\in[1,d]} \,  |\bar{x}_{ai}[j] - \bar{u}_a[j]|^{{p}} \bigg)^{\nicefrac{q}{p}} 
\, \bigg]^{\nicefrac{1}{q}}\, 
 \]
Ties are again broken by minimization of the $L(2)$ norm.
\end{enumerate} 
The {\em solution of the linear $L(p,q)$-aggregation method} is $(u_a,s_a)$. 

\subsection{Axioms and the linear functional setting} 

We now briefly discuss the properties of this method.  First, it is clear that for $p,q \in (1,\infty)$ the method is continuous, since in each one of its steps one is minimizing an unrestricted loss function that varies continuously with the data and which has a unique solution. 
To obtain consistency, we need to make an extra hypothesis about reviewers: that each criteria of evaluation has at least some importance to at least one reviewer. We say that a criteria $j\in[1,d]$ is {\em ignored} if for all reviewers, the recommendations they assign are independent of the value of the score vector for this criteria.   

\begin{proposition}
    Suppose that there are no ignored criteria. Then the linear $L(p,q)$-aggregation method satisfies 
    {\sc Consistency with score vectors} for every $p,q \in (1,\infty)$. 
\end{proposition}


Given that the method is consistent, by a reasoning similar to that of Proposition \ref{prop.consistency-notsproof}, we cannot expect it to also satisfy strategy-proofness. 

In conclusion, the linear aggregation method achieves a
 good stability and axiomatic behavior by restricting the class of functions used in explaining the behavior of reviewers. There is a price to be paid, though, which is that the learned hypothesis no longer can adjust in an exact manner to what the reviewers demand. This implies the loss of consensus.   

\begin{proposition}  
{\sc Consensus with score vectors} is not satisfied by the linear $L(p,q)$-aggregation method for any $p,q \in [1,\infty)$.  
\end{proposition} 

\paragraph{Acknowledgements.} Barcel\'o and Steifer have been funded by ANID - Millennium Science Initiative Program -
Code ICN17002. Barcel\'o and Rojas have been funded by the National Center for Artificial
Intelligence CENIA FB210017, Basal ANID.